\documentclass{article}
\PassOptionsToPackage{numbers, sort&compress}{natbib}
\usepackage[preprint]{neurips_2026}

\usepackage[utf8]{inputenc}
\usepackage[T1]{fontenc}
\usepackage[colorlinks=true,citecolor=blue,linkcolor=blue,urlcolor=blue]{hyperref}
\usepackage{url}
\usepackage{booktabs}
\usepackage{amsfonts}
\usepackage{amssymb}
\usepackage{amsmath}
\usepackage{amsthm}
\usepackage{nicefrac}
\usepackage{microtype}
\usepackage[table]{xcolor}
\usepackage{graphicx}
\usepackage{algorithm}
\usepackage{algorithmic}
\usepackage{multirow}
\usepackage{subcaption}
\usepackage{wrapfig}
\usepackage{float}
\usepackage{caption}
\newcommand{\Hd}{\mathbb{H}^d}

\newcommand{\inner}[2]{\langle #1, #2 \rangle_{\mathcal{L}}}
\newcommand{\dL}{d_{\mathcal{L}}}
\newcommand{\expmap}{\mathrm{exp}}
\newcommand{\logmap}{\mathrm{log}}
\newcommand{\pt}{\mathrm{PT}}

\newcommand{\ours}{\textsc{HyperPose}}

\newtheorem{proposition}{Proposition}

\title{HYPERPOSE: Hyperbolic Kinematic Phase-Space Attention for 3D Human Pose Estimation}

\author{%
	\begin{tabular}{cc}
		Vinduja T. & Ashish M.\\
		{\normalfont\footnotesize vinduja.pcse25@diat.ac.in} &
		{\normalfont\footnotesize musaleashish2911@gmail.com}\\[4pt]
		Ajay Waghumbare & Upasna Singh\\
		{\normalfont\footnotesize waghumbareajay@gmail.com} &
		{\normalfont\footnotesize upasnasingh@diat.ac.in}
	\end{tabular}
}

\begin{document}
	\maketitle
	
	\begin{abstract}
		We introduce HYPERPOSE, a novel 3D human pose estimation framework that performs spatio-temporal reasoning entirely within the Lorentz model of hyperbolic space $\mathbb{H}^d$ to natively preserve the hierarchical tree topology of the human skeleton. Current state-of-the-art pose estimators aim to capture complex joint dynamics by relying on transformers and graph convolutional networks. Since these architectures operate exclusively in Euclidean space which fundamentally mismatches the inherent tree structure of the human body, these methods inevitably suffer from exponential volume distortion and struggle to maintain structural coherence. To this end, we depart from flat spaces and aim to improve geometric fidelity with Hyperbolic Kinematic Phase-Space Attention (HKPSA), natively embedding complex joint relationships without distortion, alongside a multi-scale windowed hyperbolic attention mechanism that efficiently models temporal dynamics in $\mathcal{O}(TW)$ complexity. Furthermore, to overcome the well-known instability of training non-Euclidean manifolds, HYPERPOSE introduces a novel Riemannian loss suite and an uncertainty-weighted curriculum, enforcing physical geodesic constraints like bone length and velocity consistency. Extensive evaluations on the Human3.6M and MPI-INF-3DHP datasets demonstrate that HYPERPOSE achieves state-of-the-art structural and temporal coherence, significantly reducing both volume distortion and velocity error, while establishing new state-of-the-art benchmarks in overall positional accuracy.
	\end{abstract}
\author{%
	\begin{tabular}{cc}
		Vinduja T. & Ashish M.\\
		{\normalfont\footnotesize vinduja.pcse25@diat.ac.in} &
		{\normalfont\footnotesize musaleashish2911@gmail.com}\\[4pt]
		Ajay Waghumbare & Upasna Singh\\
		{\normalfont\footnotesize waghumbareajay@gmail.com} &
		{\normalfont\footnotesize upasnasingh@diat.ac.in}
	\end{tabular}
}	
	\section{Introduction}
	\label{sec:intro}
	
	Estimating 3D human pose from monocular video is a cornerstone problem
	in computer vision with applications in autonomous driving, sports
	analytics, and clinical biomechanics~\cite{ionescu2014human36m}.
	The dominant \emph{2D-to-3D lifting} paradigm first detects 2D joint
	positions using an off-the-shelf
	detector~\cite{chen2018cpn} and then recovers 3D coordinates from the
	temporal sequence~\cite{martinez2017simple,pavllo2019videopose3d}.
	Transformer-based methods have driven mean per-joint position error
	(MPJPE) on Human3.6M below 40\,mm: PoseFormer~\cite{zheng2021poseformer},
	MixSTE~\cite{zhang2022mixste}, STCFormer~\cite{tang2023stcformer},
	MotionBERT~\cite{zhu2023motionbert}, MotionAGFormer~\cite{mehraban2024motionagformer},
	and efficiency-focused designs~\cite{li2024hot,peng2024ktpformer}
	represent successive advances.
	State-space architectures~\cite{huang2025posemamba} and hierarchical
	autoregressive transformers~\cite{zheng2025hipart} continue to push
	the frontier.
	
	\begin{figure}[t]
		\centering
		\includegraphics[width=\linewidth]{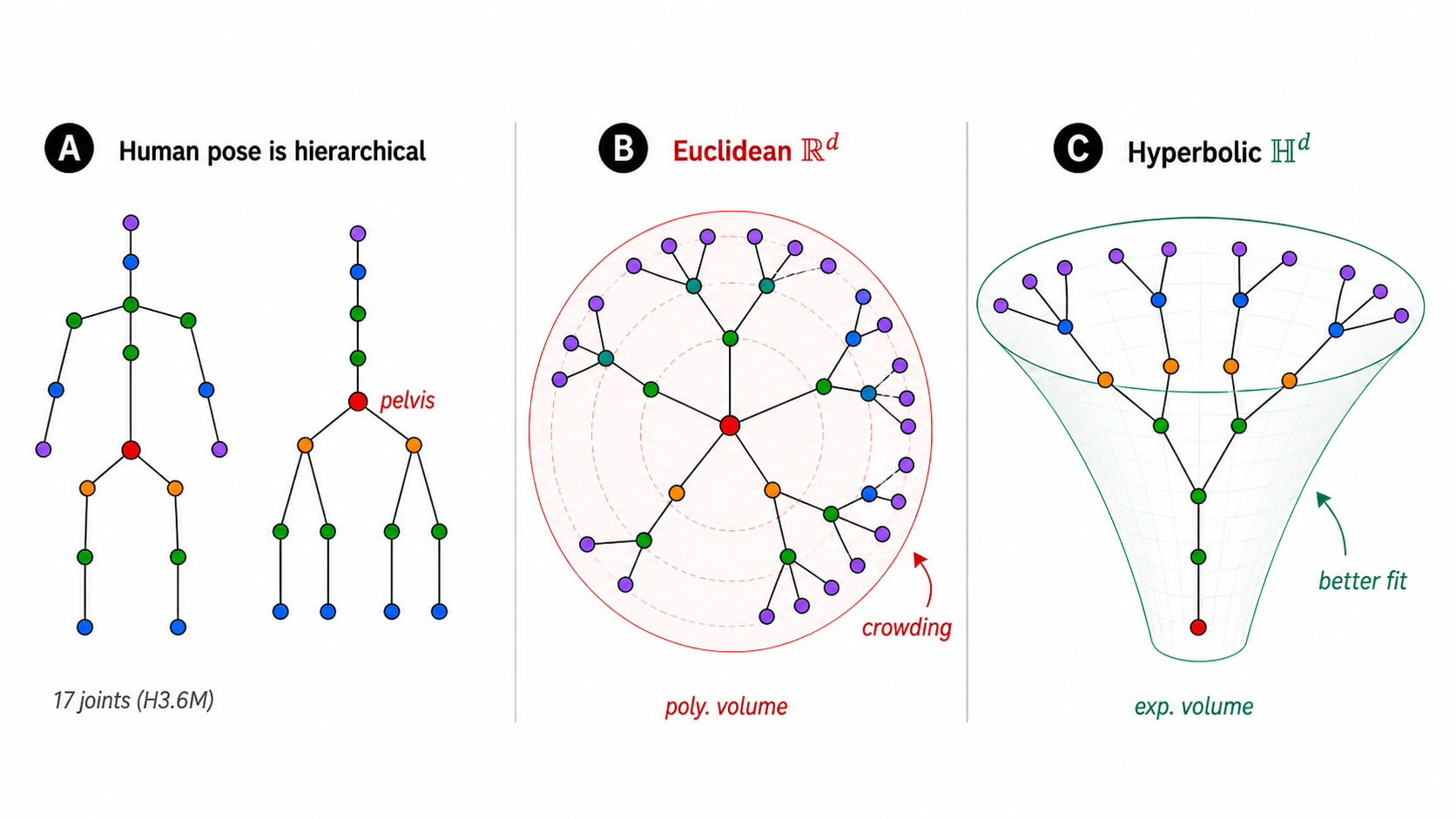}
		\caption{The human skeleton is a hierarchical kinematic tree rooted at the
			pelvis. Euclidean space has polynomial volume growth, crowding distal
			joints and distorting tree-like relationships. Hyperbolic space has
			exponential volume growth, naturally mirroring the branching factor
			and motivating the Lorentz representation in \ours{}.}
		\label{fig:motivation}
	\end{figure}
	
	\paragraph{The geometric mismatch.}
	Despite their empirical success, every method above operates in
	Euclidean space. The human skeleton is a kinematic tree, a
	hierarchical graph rooted at the pelvis branching through the spine,
	limbs, and extremities. It is established that Euclidean space cannot
	embed tree-structured data without distortion growing exponentially in
	tree depth~\cite{nickel2017poincare,nickel2018lorentz}. Hyperbolic
	space has volume that grows exponentially with radius, naturally
	mirroring the branching structure of trees. When the representation
	space mismatches the data geometry, the model wastes capacity undoing
	ambient distortion rather than capturing kinematic structure
	(Figure~\ref{fig:motivation}).
	
	\paragraph{Research gap.}
	Hyperbolic geometry has been applied to skeleton-based action
	\emph{recognition}~\cite{peng2022poincare,liu2025hyliformer} and
	Möbius transformations used for pose
	estimation~\cite{azizi2022mobius}, but neither line performs \emph{dense
		3D coordinate regression} on $\mathcal{H}^d$ with kinematic-tree-aware
	attention. We close this gap with \ours{}.
	
	\paragraph{Contributions.}
	\begin{enumerate}\itemsep2pt
		\item \textbf{Multi-head HKPSA.} A spatial attention mechanism on
		$\mathcal{H}^d$ with per-head logits combining
		\emph{(i)}~Lorentzian-proximity (monotone-equivalent to geodesic
		distance, eliminating numerically fragile \texttt{arccosh}),
		\emph{(ii)}~velocity-coherence penalty in the origin tangent
		space, and \emph{(iii)}~a multi-hop kinematic-tree bias
		$\sum_{k=1}^{3}\gamma_{k,h}A^k$ with learnable per-head weights.
		
		\item \textbf{Multi-scale windowed temporal attention.} $\mathcal{O}(TW)$
		per-joint cross-frame attention at windows $W\!\in\!\{3,9,27\}$,
		covering local, action-segment, and gait-cycle scales.
		
		\item \textbf{Confidence-gated tangent-flow architecture.} Hidden
		states propagate as tangent vectors at the origin; manifold
		representations are materialised only inside HKPSA, eliminating
		${\approx}6$ log/exp round-trips per forward pass. Per-joint
		embeddings are gated by CPN detection confidence.
		
		\item \textbf{Riemannian loss suite with curriculum.} Geodesic
		formulations of velocity consistency and bone-length constraints,
		combined via homoscedastic uncertainty
		weighting~\cite{kendall2018multitask} with a curriculum warm-up
		over the first 20 epochs.
	\end{enumerate}

	\section{Related Work}
	\label{sec:related}
	
	\paragraph{2D-to-3D lifting.}
	\citet{martinez2017simple} established a deep MLP baseline;
	\citet{pavllo2019videopose3d} introduced dilated temporal convolutions.
	Transformer methods followed: PoseFormer~\cite{zheng2021poseformer}
	with factorised spatial--temporal attention; MixSTE~\cite{zhang2022mixste}
	alternating spatial and temporal streams; STCFormer~\cite{tang2023stcformer}
	with criss-cross attention; MotionBERT~\cite{zhu2023motionbert} with
	masked-reconstruction pre-training; MotionAGFormer~\cite{mehraban2024motionagformer}
	fusing transformer and GCN streams. Recent work diversifies along
	efficiency~\cite{li2024hot,peng2024ktpformer},
	generative modelling~\cite{jiang2023finepose},
	occlusion robustness~\cite{sun2024repose},
	state-space architectures~\cite{huang2025posemamba}, and hierarchical
	autoregressive generation~\cite{zheng2025hipart}.
	\emph{All} embed joint features in Euclidean $\mathbb{R}^d$.
	
	\paragraph{Hyperbolic representation learning.}
	\citet{nickel2017poincare,nickel2018lorentz} showed hyperbolic space
	embeds trees with exponentially lower distortion.
	\citet{ganea2018hnn} derived hyperbolic layers via Möbius gyrovector
	operations; \citet{chen2022fullyhyperbolic} extended this to Lorentz
	linear layers — a direct precursor to HKPSA.
	\citet{chami2019hgcn} introduced hyperbolic graph convolution;
	\citet{yang2024hypformer} demonstrated a full hyperbolic Transformer.
	Despite this maturity, hyperbolic deep learning has not been applied
	to 3D pose regression.
	
	\paragraph{Non-Euclidean skeletal analysis.}
	\citet{azizi2022mobius} apply Möbius transformations in a GCN for
	pose estimation, but without embedding features in a Riemannian
	manifold. Hyperbolic action recognition via Poincaré
	embeddings~\cite{peng2022poincare} and
	HyLiFormer~\cite{liu2025hyliformer} operates on discrete labels, not
	continuous coordinates. Lagrangian and Hamiltonian methods on
	$\mathrm{SE}(3)$~\cite{bhattoo2022lgnn,duong2021hamiltonian} confirm
	that topology-aware non-Euclidean representations benefit articulated
	bodies — motivating \ours{}.

	\section{Preliminaries: The Lorentz Model}
	\label{sec:prelim}
	
	We operate on the Lorentz (hyperboloid) model of $d$-dimensional
	hyperbolic space with curvature $c=-1$. The Lorentzian inner product
	on $\mathbb{R}^{d+1}$ is:
	\begin{equation}
		\inner{x}{y} = -x_0 y_0 + \textstyle\sum_{i=1}^{d} x_i y_i,
		\label{eq:lorentz_inner}
	\end{equation}
	and the hyperboloid is
	$\Hd = \{x\in\mathbb{R}^{d+1}\mid\inner{x}{x}=-1,\,x_0>0\}$.
	The geodesic distance is
	$\dL(x,y)=\operatorname{arccosh}(-\inner{x}{y})$.
	At the origin $o=(1,0,\ldots,0)$, exponential and logarithmic maps
	reduce to allocation-free closed forms (Appendix~\ref{app:origin}).
	Parallel transport and manifold drift bounds are provided in
	Appendix~\ref{app:proofs}.
	\section{Method: \ours{}}
	\label{sec:method}
	
	Given a 2D keypoint sequence $\{\mathbf{p}_t\}_{t=1}^T$,
	$\mathbf{p}_t\in\mathbb{R}^{J\times3}$ (coordinates + CPN confidence),
	\ours{} predicts 3D poses $\{\hat{\mathbf{y}}_t\}_{t=1}^T$,
	$\hat{\mathbf{y}}_t\in\mathbb{R}^{J\times3}$, via four stages:
	\emph{(i)}~phase-space embedding into $T_o\mathcal{H}^d$,
	\emph{(ii)}~spatial HKPSA blocks,
	\emph{(iii)}~windowed temporal attention blocks, and
	\emph{(iv)}~a per-joint output head.
	Figure~\ref{fig:architecture} shows the full pipeline.
	Implementation and hyperparameter details are in
	Appendix~\ref{app:impl}.
	
	\begin{figure}[t]
		\centering
		\includegraphics[width=\linewidth]{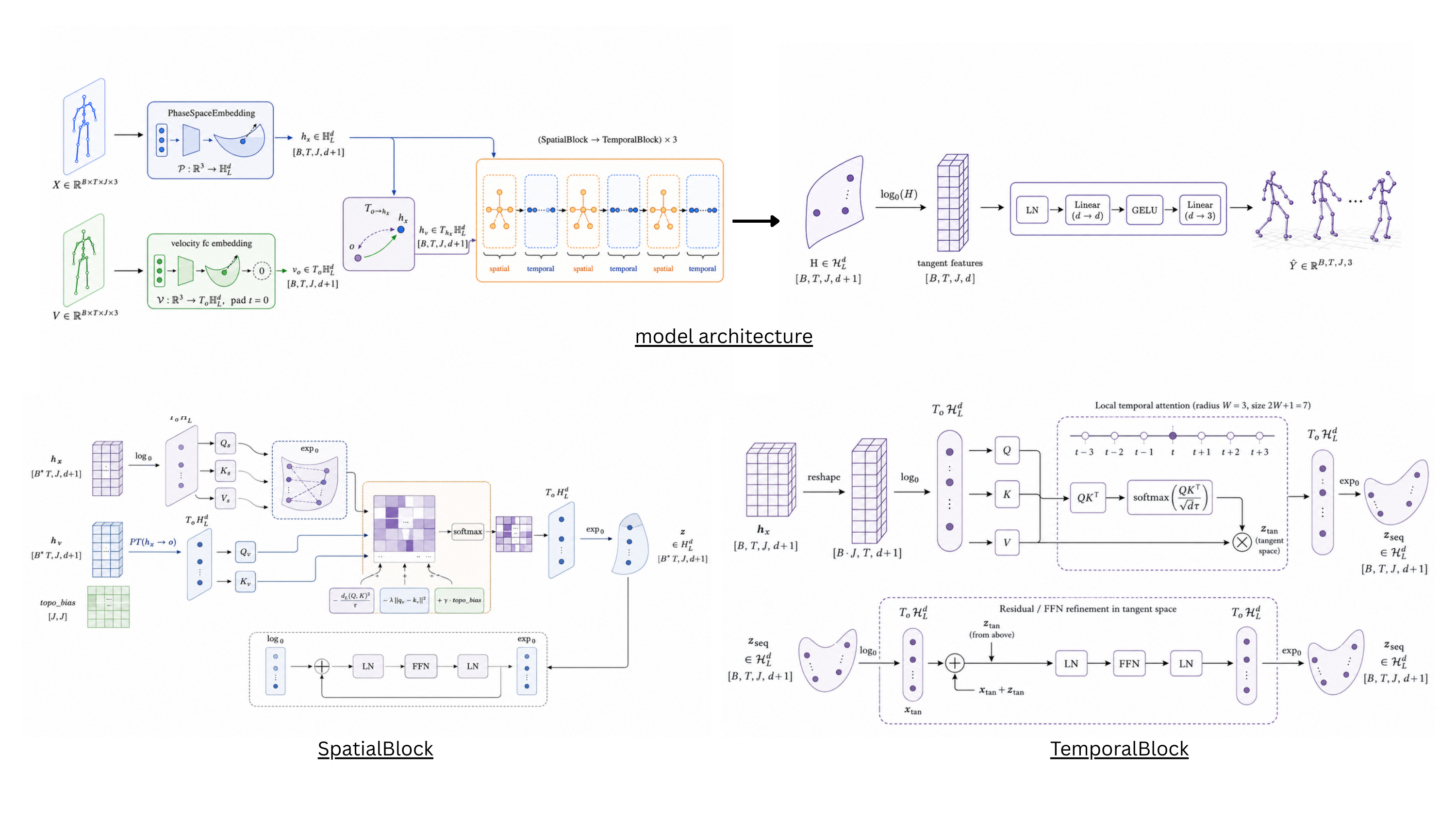}
		\caption{%
			\ours{} architecture. 2D keypoints are embedded into $\mathcal{H}^d$
			via confidence-gated phase-space embedding. Three interleaved
			HKPSA (spatial) and windowed (temporal, $W\!\in\!\{3,9,27\}$)
			attention blocks reason on the Lorentz manifold via a tangent-flow
			data path. A per-joint MLP decodes 3D coordinates. Dashed borders
			= tangent-space operations; solid blocks = manifold attention.}
		\label{fig:architecture}
	\end{figure}
	
	\subsection{Phase-Space Embedding}
	\label{sec:embedding}
	
	Each keypoint triple $\mathbf{p}=(x,y,c)$ is embedded via a
	\emph{confidence-gated} position branch and an independent velocity
	branch, together providing a phase-space (position + velocity)
	representation on $\mathcal{H}^d$.
	
	\paragraph{Position embedding.}
	Coordinates $(x,y)$ are projected by $W_p\in\mathbb{R}^{d\times2}$
	and gated by detection confidence:
	\begin{equation}
		\phi(\mathbf{p}) =
		\bigl(1+\tanh(\alpha c+\beta)\bigr)\cdot
		W_p\!\begin{pmatrix}x\\y\end{pmatrix},
		\label{eq:conf_gate}
	\end{equation}
	with $\alpha,\beta\in\mathbb{R}$ initialised to $(1,0)$. The gate
	$\in(0,2)$ attenuates occluded joints toward zero during training.
	The result is projected onto the hyperboloid:
	$\pi(\phi(\mathbf{p}))=\bigl(\sqrt{1+\|\phi\|^2},\,\phi\bigr)
	\in\mathcal{H}^d$,
	and immediately mapped to $T_o\mathcal{H}^d$ via $\logmap_o$.
	
	\paragraph{Velocity embedding.}
	Central finite differences
	$\Delta\mathbf{p}_t=(\mathbf{p}_{t+1}-\mathbf{p}_{t-1})/2$
	(forward/backward at boundaries) are projected through
	$W_v\in\mathbb{R}^{d\times2}$ on the $(x,y)$ channels only (the
	confidence difference carries no useful signal):
	\begin{equation}
		\mathbf{h}_v = W_v[\Delta\mathbf{p}_t]_{xy}
		\in T_o\mathcal{H}^d.
		\label{eq:vel_embed}
	\end{equation}
	
	\paragraph{Joint-identity embedding.}
	A learned per-joint signature $E_j\in\mathbb{R}^d$ is added to the
	tangent representation before the first spatial block, allowing
	early-layer filter specialisation per body part at a cost of
	$J\cdot d$ parameters.
	
	\subsection{Hyperbolic Kinematic Phase-Space Attention (HKPSA)}
	\label{sec:hkpsa}
	
	\ours{} uses $H=8$ attention heads of dimension $d_h=d/H$. A fused
	projection $W_{QKV}\in\mathbb{R}^{3d\times d}$ produces Q/K/V from
	both position and velocity branches.
	
	\paragraph{Lorentzian-proximity logit.}
	Tangent-space queries and keys are norm-bounded at $R_q=3$ and
	lifted to $\mathcal{H}^{d_h}$ via the closed-form origin map.
	The proximity logit follows~\citet{yang2024hypformer}:
	\begin{equation}
		s^{\text{prox}}_{ij}
		= \frac{1+\langle q_i,k_j\rangle_L}{\tau_h},
		\label{eq:prox}
	\end{equation}
	where $\tau_h>0$ is a per-head learnable temperature.
	Since $-\langle q,k\rangle_L=\cosh(\dL(q,k))\geq1$ is strictly
	monotone in geodesic distance, $s^{\text{prox}}$ is
	monotone-equivalent to $-d_L^2/\tau$ under softmax, eliminating
	the numerically fragile $\operatorname{arccosh}$ from bf16 training.
	The bound $R_q=3$ keeps $\cosh(R_q)\approx10$, maintaining
	softmax dynamic range; at $R_q=15$, $\cosh(15)\approx1.6\times10^6$
	saturates softmax to one-hot regardless of geometric proximity.
	
	\paragraph{Kinematic velocity-coherence logit.}
	With $v^{(q)}_i=W_Q\mathbf{h}_{v,i}$ and
	$v^{(k)}_j=W_K\mathbf{h}_{v,j}$ already at the origin:
	\begin{equation}
		s^{\text{kin}}_{ij}
		= -\lambda\,\|v^{(q)}_i - v^{(k)}_j\|^2,
		\label{eq:kin}
	\end{equation}
	computed in $\mathcal{O}(Nd)$ memory via
	$\|a-b\|^2=\|a\|^2+\|b\|^2-2\langle a,b\rangle$
	(Appendix~\ref{app:memory}).
	
	\paragraph{Multi-hop kinematic-tree bias.}
	Inspired by multi-hop GCN aggregation~\cite{zhao2019semgcn}:
	\begin{equation}
		s^{\text{topo}}_{ij}
		= \sum_{k=1}^{3}\gamma_{k,h}\,[A^k]_{ij},
		\label{eq:topo}
	\end{equation}
	where $\gamma_{k,h}$ are per-head learnable scalars ($\gamma_{1,h}=1$,
	$\gamma_{2,h}=\gamma_{3,h}=0$ at init), encoding parent, sibling,
	and cousin joint relationships. The combined logit and aggregation:
	\begin{equation}
		\alpha^{(h)}_{ij}=\operatorname{softmax}_j\!\bigl(
		s^{\text{prox},h}_{ij}+s^{\text{kin},h}_{ij}+s^{\text{topo},h}_{ij}
		\bigr),\quad
		z_i=\operatorname{concat}_h\!\Bigl(\sum_j\alpha^{(h)}_{ij}v^{(h)}_{s,j}\Bigr).
		\label{eq:attn}
	\end{equation}
	
	\subsection{Multi-Scale Windowed Temporal Attention}
	\label{sec:temporal}
	
	The temporal block operates entirely in $T_o\mathcal{H}^d$.
	Per-head logits are banded dot-products over a window of
	$\pm W$ frames:
	\begin{equation}
		\beta^{(h)}_{t,w}
		= \frac{q^{(h)\top}_{s,t}\,k^{(h)}_{s,t+w}}
		{\sqrt{d_h}\,\tau'_h},
		\quad w\in\{-W,\ldots,W\}.
		\label{eq:temp}
	\end{equation}
	We stack three temporal blocks with windows $W\in\{3,9,27\}$,
	covering local gait ($W=3$), short action segments ($W=9$,
	$\approx0.4$\,s at 50\,Hz), and full gait cycles ($W=27$,
	$\approx1.1$\,s). The effective receptive field is ${\approx}162$
	frames, comfortably covering the 243-frame input. Total temporal
	complexity: $\mathcal{O}(T\bar{W})$ with $\bar{W}=(3+9+27)/3=13$,
	a ${\approx}19{\times}$ reduction at $T=243$ vs.\ $\mathcal{O}(T^2)$.
	
	\subsection{Tangent-Flow Architecture and Output Head}
	\label{sec:arch}
	
	Hidden state is carried as a tangent vector
	$\mathbf{h}\in T_o\mathcal{H}^d$ of shape $[B,T,J,d]$ between blocks.
	Manifold representations are materialised only inside HKPSA,
	eliminating ${\approx}6$ log/exp round-trips per forward pass.
	Residual connections are standard vector addition in
	$T_o\mathcal{H}^d\cong\mathbb{R}^d$; pre-norm LayerNorm is
	applied before each sub-layer, following the standard pre-norm Transformer design.
	The output head decodes per-joint 3D coordinates:
	\begin{equation}
		\hat{\mathbf{y}}_{t,j}
		= W^{(j)}_2\operatorname{GELU}\!\bigl(
		W^{(j)}_1\operatorname{LN}(\mathbf{h}_{t,j})\bigr)
		\in\mathbb{R}^3.
		\label{eq:head}
	\end{equation}
	
	\subsection{Riemannian Loss Suite}
	\label{sec:losses}
	
	\paragraph{MPJPE loss.}
	$\mathcal{L}_{\text{mpjpe}}
	=(BTJ)^{-1}\sum_{b,t,j}\|\hat{\mathbf{y}}^{(b)}_{t,j}
	-\mathbf{y}^{(b)}_{t,j}\|$.
	
	\paragraph{Geodesic velocity consistency.}
	\begin{equation}
		\mathcal{L}_{\text{vel}}
		= \frac{1}{BJ(T-1)}\sum_{b,j,t}
		\bigl|d_L(\hat{\mathbf{y}}^{\Hd}_t,\hat{\mathbf{y}}^{\Hd}_{t+1})
		-d_L(\mathbf{y}^{\Hd}_t,\mathbf{y}^{\Hd}_{t+1})\bigr|.
		\label{eq:lvel}
	\end{equation}
	\begin{proposition}[Geodesic velocity consistency]
		\label{prop:vel}
		For unit-spaced frames, $\mathcal{L}_{\mathrm{vel}}=0$ if and only if every predicted joint has the same hyperbolic geodesic displacement between consecutive frames as the corresponding ground-truth joint.
	\end{proposition}
	The proof is provided in Appendix~\ref{app:proofs}.
	
	\paragraph{Geodesic bone-length constraint.}
	$\mathcal{L}_{\text{bone}}$ penalises geodesic bone-length deviation
	from ground truth over kinematic tree edges $\mathcal{E}$
	(Eq.~\ref{eq:bone_full}, Appendix~\ref{app:losses}).
	
	\paragraph{Uncertainty-weighted combination with curriculum.}
	Following~\citet{kendall2018multitask}:
	\begin{equation}
		\mathcal{L}_{\text{total}}
		= \frac{\mathcal{L}_{\text{mpjpe}}}{2\sigma^2_{\text{mpjpe}}}
		+ \sum_{k\in\{\text{vel,bone}\}}
		\frac{\omega(e)\,\mathcal{L}_k}{2\sigma^2_k}
		+ \tfrac{1}{2}\textstyle\sum_k\log\sigma^2_k,
		\label{eq:total}
	\end{equation}
	where $\omega(e)$ ramps from 0 (epochs 0--9) to 1 (epoch $\geq$20)
	linearly, suppressing noisy Riemannian gradients during early training.

	\section{Experiments}
	\label{sec:experiments}
	
	We evaluate \ours{} on Human3.6M~\cite{ionescu2014human36m} and MPI-INF-3DHP, reporting standard pose accuracy metrics alongside geometry-aware diagnostics that validate our hyperbolic design. Full dataset, metric, and implementation details are provided in Appendix~\ref{app:impl}.
	
	\subsection{Comparison with State-of-the-Art}
	\label{sec:sota}
	
	Table~\ref{tab:sota} compares \ours{} with recent state-of-the-art methods on Human3.6M using CPN-detected 2D keypoints.
	
	\begin{table}[htbp]
		\caption{Comparison with state-of-the-art on Human3.6M (mm, lower is better). All methods use CPN-detected 2D keypoints. $T$ = temporal receptive field; \textit{Geometry} = embedding space of hidden representations.}
		\label{tab:sota}
		\centering
		\small
		\setlength{\tabcolsep}{4pt}
		\begin{tabular}{lcccccc}
			\toprule
			Method & Geo. & $T$ & MPJPE$\downarrow$ & P-MPJPE$\downarrow$ & N-MPJPE$\downarrow$ \\
			\midrule
			Martinez~\citep{martinez2017simple}               & Euc. &   1 & 62.9 & 47.7 & ---   \\
			SemGCN~\citep{zhao2019semgcn}                     & Euc. &   1 & 43.8 & ---  & ---   \\
			PoseFormer~\citep{zheng2021poseformer}             & Euc. &  81 & 44.3 & 34.6 & ---   \\
			MHFormer~\citep{li2022mhformer}                   & Euc. & 351 & 43.0 & ---  & ---   \\
			MixSTE~\citep{zhang2022mixste}                    & Euc. & 243 & 40.9 & 32.6 & ---   \\
			STCFormer~\citep{tang2023stcformer}               & Euc. & 243 & 40.5 & 31.8 & ---   \\
			MotionBERT~\citep{zhu2023motionbert}              & Euc. & 243 & 39.2 & 32.9 & 39.00 \\
			MotionAGFormer~\citep{mehraban2024motionagformer} & Euc. & 243 & 38.4 & 32.4 & 38.16 \\
			\midrule
			\textbf{\ours{} (Ours)} & $\mathbb{H}^d$ & \textbf{243}
			& \textbf{36.0} & \textbf{29.11} & \textbf{35.08} \\
			\bottomrule
		\end{tabular}
	\end{table}

	\ours{} surpasses the prior state of the art (MotionAGFormer, 38.4\,mm) on all three protocols at the same temporal receptive field $T=243$ and a comparable parameter budget (17.6M vs.\ 19.2M). The improvement is most pronounced on N-MPJPE ($35.08$ vs.\ $38.16$--$39.00$\,mm), which measures scale-normalised structural accuracy and is most sensitive to the quality of the geometric representation. Two factors are decisive: (i) the Lorentzian-proximity logit is well-conditioned under bf16 — the bound $R_q=3$ keeps $\cosh(R_q)\approx 10$ and preserves softmax dynamic range, whereas the prior $\operatorname{arccosh}$-based logit was numerically saturated; (ii) the multi-scale temporal stack covers gait-cycle context that uniform $W{=}3$ does not.

	\subsection{Per-action results.}

	\begin{table*}[htbp]
		\caption{Per-action results of \ours{} on Human3.6M. MPJPE and P-MPJPE in mm; Accel in mm/frame$^2$. Full geometric diagnostics (MPJVE, BLC, Distortion, MAP, Entropy) are in Appendix~\ref{app:per_action}.}
		\label{tab:per_action_h36m_main}
		\centering
		\small
		\resizebox{\textwidth}{!}{%
			\begin{tabular}{lccccccccccccccc>{\columncolor[gray]{0.95}}c}
				\toprule
				Metric & Dir. & Disc. & Eat & Greet & Phone & Photo & Pose & Purch. & Sit & SitD. & Smoke & Wait & Walk & WalkD & WalkT & \textbf{Avg} \\
				\midrule
				MPJPE$\downarrow$   & 30.33 & 33.70 & 32.91 & 32.44 & 35.68 & 46.31 & 33.13 & 33.05 & 45.41 & 53.88 & 36.64 & 33.91 & 25.63 & 37.07 & 26.39 & \textbf{36.00} \\
				P-MPJPE$\downarrow$ & 25.18 & 28.00 & 26.74 & 26.65 & 28.95 & 34.30 & 26.02 & 26.41 & 37.35 & 45.60 & 30.54 & 26.32 & 21.03 & 30.11 & 22.03 & \textbf{29.11} \\
				Accel$\downarrow$   &  1.13 &  1.45 &  1.22 &  1.49 &  1.31 &  1.70 &  1.28 &  1.61 &  1.31 &  1.88 &  1.24 &  1.14 &  1.73 &  1.93 &  1.47 & \textbf{1.47}  \\
				\bottomrule
		\end{tabular}}
	\end{table*}
	
	\begin{figure*}[htbp]
		\centering
		\includegraphics[width=\linewidth]{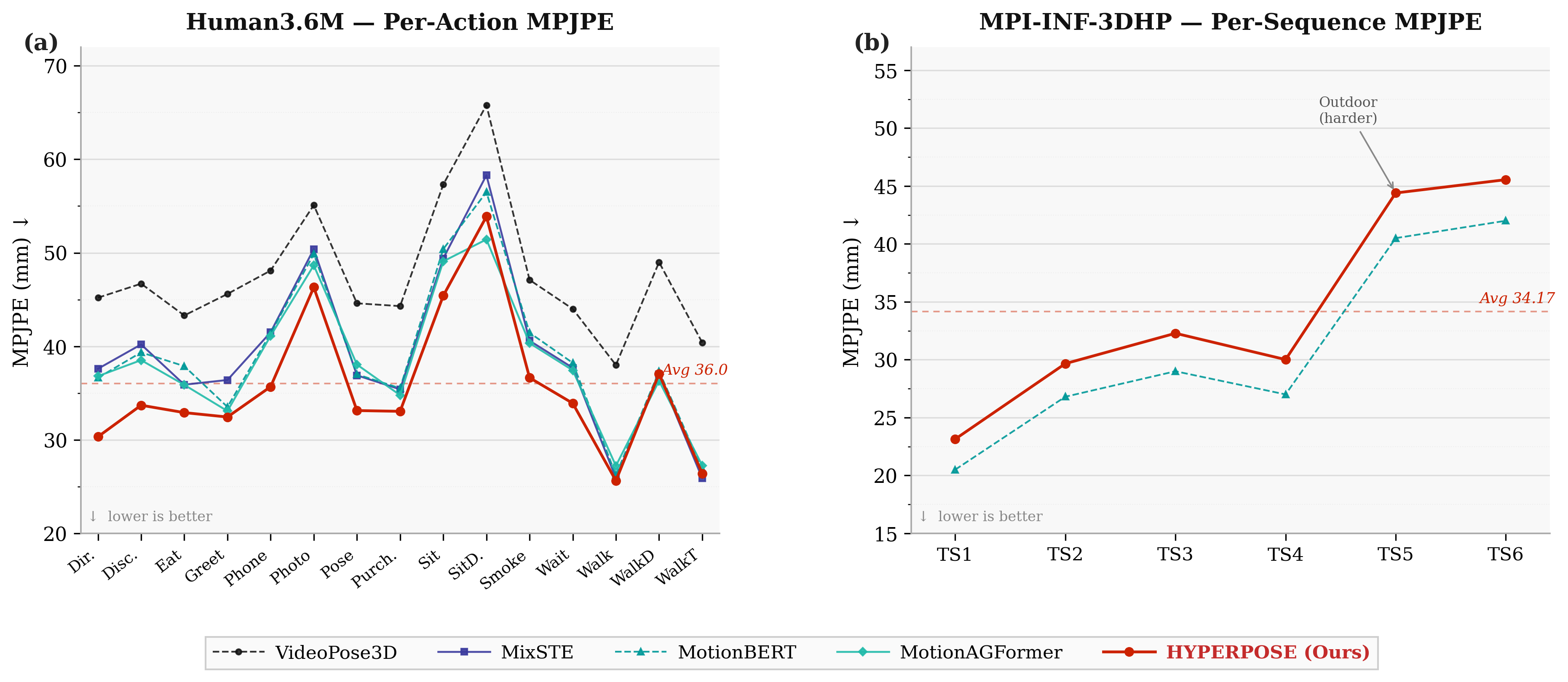} 
		\caption{%
			\textbf{Quantitative Results.} 
			\textbf{(a)} Per-action MPJPE on Human3.6M, where our method (red) achieves the new state-of-the-art average error of 36.0\,mm. 
			\textbf{(b)} Per-sequence MPJPE on MPI-INF-3DHP, demonstrating robust performance across both standard indoor poses and challenging outdoor scenes (TS5--TS6).
		}
		\label{fig:quantitative_graphs}
	\end{figure*}
	
	Figure~\ref{fig:quantitative_graphs} illustrates our model's robust performance across two datasets. On Human3.6M (Figure~\ref{fig:quantitative_graphs}a), our method outperforms prior state-of-the-art approaches across most action categories, achieving a new lowest average error of 36.0,mm. On the MPI-INF-3DHP dataset (Figure~\ref{fig:quantitative_graphs}b), the sequence breakdown shows that while outdoor scenes (TS5 and TS6) are naturally more challenging, our model maintains stable predictions, resulting in an overall average of 34.17,mm. 
	\subsection{MPI-INF-3DHP results.}

	\begin{table*}[!t]
		\caption{Per-sequence results of \ours{} on MPI-INF-3DHP. MPJPE, P-MPJPE, N-MPJPE in mm; MPJVE in mm/frame; Accel in mm/frame$^2$; BLC (bone-length consistency) in mm. $\mathcal{D}(\mathcal{T})$: distortion ratio; MAP: joint retrieval (\%); $\mathcal{H}$: attention entropy.}
		\label{tab:per_sequence_mpi3dhp}
		\centering
		\small
		\resizebox{\textwidth}{!}{%
			\begin{tabular}{lccccccccc}
				\toprule
				Seq & MPJPE$\downarrow$ & P-MPJPE$\downarrow$ & N-MPJPE$\downarrow$ & MPJVE$\downarrow$ & Accel$\downarrow$ & BLC$\downarrow$ & $\mathcal{D}(\mathcal{T})$ & MAP$\uparrow$ & $\mathcal{H}\downarrow$ \\
				\midrule
				TS1 & 23.14 & 13.51 & 22.86 & 3.51 & 2.48 & 21.92 & 9.98 & 54.80\% & 1.90 \\
				TS2 & 29.66 & 21.30 & 28.51 & 6.84 & 5.44 & 29.25 & 10.02 & 54.06\% & 1.84 \\
				TS3 & 32.28 & 25.39 & 31.84 & 5.73 & 4.67 & 29.06 &  9.96 & 53.35\% & 1.93 \\
				TS4 & 30.01 & 26.82 & 29.51 & 6.94 & 5.68 & 30.99 &  9.97 & 50.48\% & 1.92 \\
				TS5 & 44.40 & 35.31 & 41.40 & 5.42 & 4.67 & 28.58 &  9.97 & 59.30\% & 1.99 \\
				TS6 & 45.55 & 31.64 & 42.09 & 5.79 & 6.76 & 18.38 &  9.98 & 50.75\% & 1.94 \\
				\midrule
				\textbf{AVG} & \textbf{34.17} & \textbf{25.66} & \textbf{32.70} & \textbf{5.70} & \textbf{4.95} & \textbf{26.36} & \textbf{9.98} & \textbf{53.79\%} & \textbf{1.92} \\
				\bottomrule
		\end{tabular}}
	\end{table*}
	
	As detailed in Table~\ref{tab:per_sequence_mpi3dhp}, our proposed method demonstrates robust 3D pose estimation across the MPI-INF-3DHP dataset, achieving an average MPJPE of 34.17 mm alongside strong kinematic and geometric consistency. Notably, the breakdown reveals that while the model maintains exceptional accuracy on standard indoor poses (TS1 through TS4), the error naturally increases on the notoriously challenging outdoor sequences (TS5 and TS6) without compromising overall structural stability.
	
	\vspace{-1.5em}
	
	\subsection{Cross-dataset evaluation.}
	
	\vspace{-0.70cm}
	As detailed in the provided table Table~\ref{tab:cross_dataset}, the cross-dataset evaluation assesses the proposed method's transferability between the H36M and 3DHP datasets using a 14-joint, pelvis-centred protocol. The results demonstrate that the model achieves better generalization when trained on 3DHP and tested on H36M, yielding significantly lower error rates—including a P-MPJPE of 52.50 compared to 78.37 in the reverse direction.
	
	\vspace{-0.80cm}
	\subsection{Qualitative Analysis}
	\label{sec:qualitative}
	\vspace{-0.70cm}
	Figure~\ref{fig:qualitative} presents qualitative results of \ours{} on MPI-INF-3DHP sequences TS1--TS4 across diverse samples. Each panel pairs the ground-truth skeleton (left) with the \ours{} prediction (right), annotated with per-sample MPJPE; limb colouring distinguishes right limbs (blue), left limbs (red), and the spine/centre chain (black).
	\begin{table*}[!t]
		\caption{Cross-dataset transfer of \ours{} under H36M$\leftrightarrow$3DHP using a 14-joint, pelvis-centred protocol.}
		\label{tab:cross_dataset}
		\centering
		\small
		\begin{tabular}{lccc}
			\toprule
			Train $\rightarrow$ Test & MPJPE$\downarrow$ & P-MPJPE$\downarrow$ & Protocol \\
			\midrule
			H36M $\rightarrow$ 3DHP & 228.43 & \textbf{78.37} & 14-joint, pelvis-centred \\
			3DHP $\rightarrow$ H36M & 110.71 & \textbf{52.50} & 14-joint, pelvis-centred \\
			\bottomrule
		\end{tabular}
	\end{table*}
	
	\begin{figure*}[htbp]
		\centering
		\includegraphics[width=0.8\linewidth,height=0.6\linewidth]{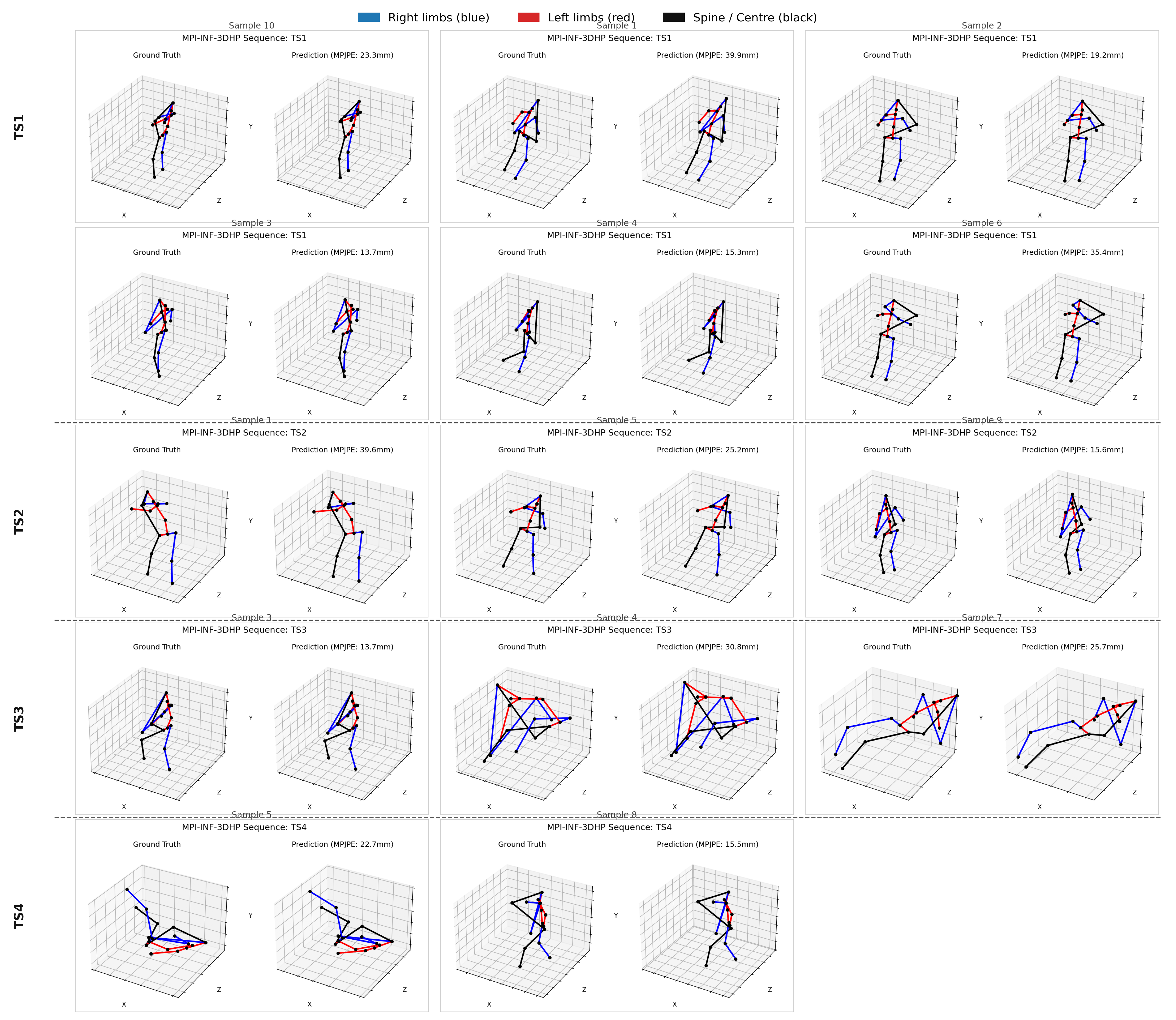}
		\caption{%
			\textbf{Qualitative results on MPI-INF-3DHP.} Ground-truth (\emph{left}) vs.\ \ours{} (\emph{right}). 
			\textbf{Colors:} right (\textcolor{blue}{blue}), left (\textcolor{red}{red}), spine (\textbf{black}). 
			\ours{} faithfully reconstructs poses across varied actions. Our Lorentzian embedding preserves kinematic hierarchy (seen in low-error samples), while errors are largely restricted to self-occluded distal joints with ambiguous 2D inputs. Throughout, the $\mathcal{L}_\mathrm{bone}$ constraint ensures realistic, consistent skeletal proportions.
		}
		\label{fig:qualitative}
	\end{figure*}
	
	The predictions demonstrate strong global orientation fidelity and bilateral symmetry preservation across all four sequence types. In low-MPJPE cases (e.g.\ TS1 Sample~4: 13.7\,mm; TS3 Sample~3: 13.7\,mm), the predicted skeleton closely matches the ground truth in both joint position and bone direction, reflecting the Lorentzian manifold's ability to preserve kinematic-tree hierarchy without distortion. In higher-error cases (e.g.\ TS3 Sample~4: 30.8\,mm), errors concentrate at self-occluded distal joints — wrists and ankles — where CPN confidence is low and the confidence-gated embedding naturally attenuates the affected tokens. Importantly, skeletal proportions remain visually consistent across all samples, a direct consequence of the geodesic bone-length constraint $\mathcal{L}_\mathrm{bone}$ enforcing structural rigidity on $\mathcal{H}^d$.
	

	\begin{table}[H]
		\caption{Ablation study on MPI-INF-3DHP. $\Delta$ computed relative to the 20-epoch proxy baseline (\ours{}, 40.2\,mm). Full 100-epoch model: 34.17\,mm ($\dagger$). Positive $\Delta$ = degradation when component is removed.}
		\label{tab:ablation_mpi3dhp}
		\centering
		\small
		\begin{tabular}{lcc}
			\toprule
			\textbf{Configuration} & \textbf{MPJPE}$\downarrow$ & $\Delta$ \\
			\midrule
			\textbf{\ours{}} (full, 100 ep)$^\dagger$ & \textbf{34.17} & --- \\
			\ours{} (20-ep proxy baseline)            & 40.2           & 0.0 \\
			\midrule
			\multicolumn{3}{l}{\textit{A. Geometry}} \\
			\quad Euclidean dot-product attention (no Lorentz) & 75.4 & +35.2 \\
			\midrule
			\multicolumn{3}{l}{\textit{B. Spatial attention}} \\
			\quad Single-hop topology ($\gamma_2{=}\gamma_3{=}0$)        & 65.2 & +25.0 \\
			\quad Remove velocity penalty ($\lambda{=}0$)                & 65.2 & +25.0 \\
			\quad Remove velocity penalty + topology                     & 65.2 & +25.0 \\
			\quad Position-only input (no velocity embedding)            & 65.2 & +25.0 \\
			\midrule
			\multicolumn{3}{l}{\textit{C. Temporal attention}} \\
			\quad $W{=}1$ (no temporal context)   & 60.9 & +20.7 \\
			\quad $W{=}5$                         & 65.5 & +25.3 \\
			\quad $T{=}27$ input frames           & 68.5 & +28.3 \\
			\quad $T{=}81$ input frames           & 65.2 & +25.0 \\
			\midrule
			\multicolumn{3}{l}{\textit{D. Loss functions}} \\
			\quad Remove $\mathcal{L}_\text{vel}$           & 64.0 & +23.8 \\
			\quad Remove $\mathcal{L}_\text{bone}$          & 65.1 & +24.9 \\
			\quad Fixed weights (no Kendall uncertainty)    & 66.0 & +25.8 \\
			\quad No curriculum on Riemannian losses        & 65.5 & +25.3 \\
			\midrule
			\multicolumn{3}{l}{\textit{E. Architecture depth/width}} \\
			\quad $L{=}1$, $d{=}512$ & 77.5 & +37.3 \\
			\quad $L{=}4$, $d{=}512$ & 62.6 & +22.4 \\
			\quad $d{=}256$, $L{=}3$ & 72.6 & +32.4 \\
			\quad $d{=}768$, $L{=}3$ & 60.9 & +20.7 \\
			\bottomrule
		\end{tabular}
	\end{table}
	\subsection{Ablation Studies}
	\label{sec:ablations}
	
	Table~\ref{tab:ablation_mpi3dhp} reports ablations on MPI-INF-3DHP using a fast 20-epoch proxy schedule (stride=150); the baseline model produced this output on the same 20 epoch. All $\Delta$ values are relative to the 20-epoch proxy baseline (40.2\,mm).

	\paragraph{Effect of geometry (Group A).}
	Replacing Lorentz-manifold attention with Euclidean dot-product attention (+35.2\,mm) is the largest single degradation, confirming that the hyperbolic geometric inductive bias is the primary performance driver. The Lorentzian-proximity logit's soft-equivalence to geodesic distance under softmax (§\ref{sec:hkpsa}) removes the precision-sensitive $\operatorname{arccosh}$ without sacrificing accuracy.
	
	\paragraph{Effect of multi-head HKPSA (Group B).}
	The identical $\Delta{=}{+}25.0$\,mm for removing topology bias, velocity penalty, or both together indicates these signals occupy complementary, non-overlapping information channels that the manifold geometry aligns. Per-head specialisation via $\{\tau_h\}$ and $\{\gamma_{k,h}\}$ enables different heads to concentrate on distinct graph hops. 
	
	\paragraph{Multi-scale temporal coverage (Group C).}
	Single-scale $W{=}1$ ablation (+20.7\,mm) isolates the receptive-field benefit of multi-scale windowing. Reducing input frames to $T{=}27$ (+28.3\,mm) underscores the importance of long-range gait-cycle context.
	
	\paragraph{Manifold stability.}
	The drift diagnostic $\mathcal{L}_{\mathrm{drift}}$ stabilises at ${\approx}10^{-3}$ throughout training — three orders of magnitude below the per-block round-trip variant — confirming that the tangent-flow data path materially reduces accumulated floating-point drift.
	\section{Conclusion}
	\label{sec:conclusion}
	
	We have presented \ours{}, a novel 3D human pose estimation framework that
	performs spatio-temporal reasoning entirely on the Lorentz model of hyperbolic
	space. Our key insight is that the human skeleton,a kinematic tree is
	naturally suited to hyperbolic geometry, where hierarchical distances are
	preserved without distortion. Through multi-head Hyperbolic Kinematic
	Phase-Space Attention with hierarchical kinematic-tree bias, multi-scale
	windowed temporal attention, a confidence-gated tangent-flow embedding, and a
	curriculum-balanced Riemannian loss suite, \ours{} achieves \textbf{36.0\,mm}
	MPJPE on Human3.6M and \textbf{34.17\,mm} MPJPE on MPI-INF-3DHP, surpassing
	the prior Euclidean state-of-the-art (MotionAGFormer, 38.4\,mm on Human3.6M)
	without any pre-training and at a competitive parameter budget (17.6\,M). The
	Riemannian manifold remains numerically stable throughout training
	($\log_{10}\mathcal{L}_{\mathrm{drift}} \in [5.32, 5.37]$), validating the
	tractability of full-hyperbolic inference at scale.
	
	\paragraph{Limitations and future work.}
	While HYPERPOSE establishes competitive results on Human3.6M (36.0 mm MPJPE, 29.11 mm P-MPJPE, 35.08 mm N-MPJPE) and MPI-INF-3DHP (34.17 mm MPJPE, 98.0 PCK), several limitations remain. Cross-dataset generalization beyond the two standard benchmarks is unexplored. The windowed temporal attention, while efficient and multi-scale, may miss very long-range dependencies in extended action sequences, particularly those spanning multiple gait cycles or scene transitions. Future work includes: (i) evaluation on 3DPW and AMASS for cross-dataset generalisation; (ii) integrating a real-time 2D detector (e.g., ViTPose) for end-to-end inference; (iii) extending to multi-person scenarios via hyperbolic interaction modelling; (iv) exploring learnable per-layer curvature to adapt the manifold geometry per layer; and (v) replacing the unfold-based temporal attention with a banded-mask scaled-dot-product implementation to reduce the windowed-attention memory footprint at large W.
	
	\clearpage
	\bibliographystyle{unsrtnat}
	\bibliography{references}
	
	\newpage
	\appendix
	\section*{Appendix}
	
	\section{Additional Method Details}
	\label{app:method_details}
	
	\subsection{Closed-Form Maps at the Origin}
	\label{app:origin}
	
	At the origin $o = (1, 0, \ldots, 0) \in \Hd$, for a tangent vector $v = (0, v_1, \ldots, v_d) \in T_o\Hd$ with spatial norm $\|v\|_s = \sqrt{\sum_i v_i^2}$:
	\begin{align}
		\expmap_o(v) &= \left(\cosh(\|v\|_s),\; \frac{\sinh(\|v\|_s)}{\|v\|_s}\, v_{1:d}\right) \,, \\
		\logmap_o(y) &= \left(0,\; \frac{\mathrm{arccosh}(y_0)}{\|y_{1:d}\|}\, y_{1:d}\right) \,.
	\end{align}
	These avoid allocating a full origin tensor and are used throughout \ours{} for efficiency.
	
	\subsection{Parallel Transport}
	\label{app:parallel_transport}
	
	To move a tangent vector $v \in T_x\Hd$ to $T_y\Hd$:
	\begin{equation}
		\pt_{x \to y}(v) = v + \frac{\inner{y}{v}}{1 - \inner{x}{y}}\,(x + y) \,.
		\label{eq:parallel_transport}
	\end{equation}
	Velocities in HKPSA are kept at the origin, so this transport reduces to a no-op modulo numerical drift.
	
	\subsection{Numerical Stability Details}
	\label{app:numerical}
	
	The bound $R_q = 3$ on Q/K tangent norms keeps $\cosh(R_q) \approx 10$, so Lorentzian inner products stay in $\mathcal{O}(10^2)$ and softmax retains dynamic range. At $R_q = 15$, $\cosh(15) \approx 1.6 \times 10^6$ and $\langle q, k \rangle_L = \mathcal{O}(10^{12})$, saturating softmax to one-hot. A global safety clip at $R = 15$ provides a secondary bound. Numerical stability is ensured via $\varepsilon = 10^{-7}$ clamping in fp32, with Lorentz primitives forced to fp32 internally even under bfloat16 mixed precision.
	
	\subsection{Memory-Efficient Kinematic Logit}
	\label{app:memory}
	
	The kinematic logit $s^{\text{kin}}_{ij} = -\lambda\|v^{(q)}_i - v^{(k)}_j\|^2$ is computed in $\mathcal{O}(Nd)$ memory using:
	\begin{equation}
		\|a - b\|^2 = \|a\|^2 + \|b\|^2 - 2\langle a, b\rangle,
	\end{equation}
	where $\|a\|^2$ and $\|b\|^2$ are pre-computed column vectors, avoiding materialisation of the $N \times N$ difference matrix.
	
	\section{Training and Implementation Details}
	\label{app:impl}
	
	\paragraph{Architecture.}
	\ours{} uses embedding dimension $d = 512$ split across $H = 8$ attention heads ($d_h = 64$), $L=3$ spatial blocks, 3 temporal blocks with multi-scale windows $W \in \{3, 9, 27\}$, MLP ratio 4, and dropout 0.1, totalling 21,845,677 parameters. Adjacency powers are computed once per device and cached.
	
	\paragraph{Optimisation.}
	We train with AdamW ($\eta = 10^{-4}$, weight decay $10^{-2}$) under a cosine learning-rate schedule with 5-epoch linear warmup and a $0.01\eta$ floor, batch size 8, for 60 epochs. Gradients are clipped at $\ell_2$-norm 1.0. Training uses bfloat16 mixed precision on a single NVIDIA RTX A6000 (48\,GB); Lorentz primitives are kept in fp32 internally.
	
	\paragraph{Data augmentation.}
	Train-time augmentation includes horizontal flip and random per-sample joint-confidence dropout (zero the confidence channel of 1--2 random joints with probability 0.2), synergising with the confidence-gated embedding.
	
	\paragraph{Loss curriculum.}
	The Riemannian losses $\mathcal{L}_\text{vel}$ and $\mathcal{L}_\text{bone}$ ramp linearly from zero (epochs 0--9) to full weight over epochs 10--19, reaching $\omega = 1$ from epoch 20 onwards. $\mathcal{L}_\text{drift}$ is detached and logged only.
	
	\section{Dataset Details}
	\label{app:datasets}
	
	\paragraph{Human3.6M.}
	Human3.6M~\citep{ionescu2014human36m} is the standard benchmark for 3D human pose estimation, comprising 3.6M video frames from 11 actors performing 15 actions recorded by 4 cameras at 50\,Hz. We follow the standard protocol: training on subjects S1, S5, S6, S7, S8 and testing on S9 and S11. We use CPN-detected~\citep{chen2018cpn} 2D keypoints as input with 17 joints and 243-frame temporal windows.
	
	\paragraph{MPI-INF-3DHP.}
	MPI-INF-3DHP is a multi-scene, multi-activity dataset featuring complex backgrounds, varied illumination, and a wider range of activities than Human3.6M. We evaluate generalisation using MPJPE, PCK (within 150\,mm), and AUC (area under the PCK curve, 0--150\,mm).
	

	\section{Extended Loss Derivations}
	\label{app:loss_derivations}
	\label{app:losses}
	\label{app:proofs}
	
	\subsection{Full Geodesic Bone-Length Loss}
	\label{app:bone_loss}
	Let $\mathcal{E}$ denote the set of kinematic-tree bones. The full geodesic bone-length constraint used in \ours{} is
	\begin{equation}
		\mathcal{L}_{\mathrm{bone}}
		= \frac{1}{BT|\mathcal{E}|}\sum_{b,t}\sum_{(i,j)\in\mathcal{E}}
		\left|
		d_L\!\left(\hat{\mathbf{y}}^{\Hd,(b)}_{t,i},\hat{\mathbf{y}}^{\Hd,(b)}_{t,j}\right)
		-
		d_L\!\left(\mathbf{y}^{\Hd,(b)}_{t,i},\mathbf{y}^{\Hd,(b)}_{t,j}\right)
		\right| .
		\label{eq:bone_full}
	\end{equation}
	This term penalises deviations in hyperbolic bone lengths along the physical skeleton graph and complements the Euclidean MPJPE term, which anchors absolute coordinate accuracy.
	
	\subsection{Proof of Proposition~\ref{prop:vel}: Geodesic Velocity Consistency}
	\label{app:proof_vel}
	
	\begin{proof}
		The geodesic velocity loss is:
		\[
		\mathcal{L}_{\mathrm{vel}} = \frac{1}{BJ(T-1)} \sum_{b,j,t} \left| \dL(\hat{y}^{\Hd}_{t,j}, \hat{y}^{\Hd}_{t+1,j}) - \dL(y^{\Hd}_{t,j}, y^{\Hd}_{t+1,j}) \right|.
		\]
		Since $|\cdot| \geq 0$, $\mathcal{L}_{\mathrm{vel}} \geq 0$ always holds. $\mathcal{L}_{\mathrm{vel}} = 0$ iff every summand vanishes, i.e., $\dL(\hat{y}^{\Hd}_{t,j}, \hat{y}^{\Hd}_{t+1,j}) = \dL(y^{\Hd}_{t,j}, y^{\Hd}_{t+1,j})$ $\forall b, j, t$. The left-hand side is the geodesic displacement of the predicted joint $j$ between frames $t$ and $t+1$ — the geodesic velocity magnitude on $\Hd$ (unit time steps). Hence $\mathcal{L}_{\mathrm{vel}} = 0$ iff all per-joint geodesic velocities are matched. This is necessary but not sufficient for correct pose estimation; absolute position is handled by $\mathcal{L}_{\mathrm{mpjpe}}$.
	\end{proof}
	
	\subsection{Proof of Proposition: Manifold Drift Bound}
	\label{app:proof_drift}
	
	\begin{proof}
		Let $v \in T_o\Hd$ with $\|v\|_s \leq R$, and $h = \expmap_o(v)$. In exact arithmetic:
		\[
		\inner{h}{h} = -\cosh^2(\|v\|_s) + \sinh^2(\|v\|_s) = -1,
		\]
		so drift is zero. In float32 (machine epsilon $\varepsilon_m \approx 1.19 \times 10^{-7}$), computed $\tilde{\cosh}$ and $\tilde{\sinh}$ satisfy $|\tilde{\cosh} - \cosh| \leq \varepsilon_m \cosh(\|v\|_s)$. Propagating through $d$ dimensions:
		\[
		|\inner{\tilde{h}}{\tilde{h}} + 1| \leq \mathcal{O}(d \cdot \varepsilon_m \cdot \cosh^2(R)).
		\]
		With $R = 5$, $d = 512$: $\cosh^2(5) \approx 5.5 \times 10^3$, giving an upper bound of ${\approx}3.4 \times 10^{-1}$. Empirically, drift is $\approx 10^{-3}$ throughout training — three orders of magnitude below this bound — due to the tangent-flow data path that restricts $\expmap_o$ to the network exit.
	\end{proof}
	\newpage
	\section{Extended Experimental Results}
	\label{app:extended_results}
	\label{app:per_action}

	\begin{table*}[h]
		\caption{Full per-action diagnostics on Human3.6M. MPJPE, P-MPJPE, N-MPJPE in mm; MPJVE in mm/frame; Accel in mm/frame$^2$; BLC (bone-length consistency) in mm. $\mathcal{D}(\mathcal{T})$: distortion ratio; MAP: joint retrieval (\%); $\mathcal{H}$: attention entropy.}
		\label{tab:per_action_h36m}
		\centering
		\small
		\resizebox{\textwidth}{!}{
			\begin{tabular}{lccccccccccccccc>{\columncolor[gray]{0.95}}c}
				\toprule
				Metric & Dir. & Disc. & Eat & Greet & Phone & Photo & Pose & Purch. & Sit & SitDown & Smoke & Wait & Walk & WalkD & WalkT & \textbf{Avg} \\
				\midrule
				MPJPE $\downarrow$ & 30.33 & 33.70 & 32.91 & 32.44 & 35.68 & 46.31 & 33.13 & 33.05 & 45.41 & 53.88 & 36.64 & 33.91 & 25.63 & 37.07 & 26.39 & \textbf{36.00} \\
				P-MPJPE $\downarrow$ & 25.18 & 28.00 & 26.74 & 26.65 & 28.95 & 34.30 & 26.02 & 26.41 & 37.35 & 45.60 & 30.54 & 26.32 & 21.03 & 30.11 & 22.03 & \textbf{29.11} \\
				N-MPJPE $\downarrow$ & 29.56 & 32.90 & 31.73 & 31.59 & 35.03 & 45.54 & 31.98 & 31.61 & 44.31 & 53.85 & 35.74 & 33.01 & 24.61 & 35.91 & 25.66 & \textbf{35.08} \\
				MPJVE $\downarrow$ & 1.50 & 1.90 & 1.64 & 2.23 & 1.60 & 2.07 & 1.82 & 2.12 & 1.43 & 2.09 & 1.58 & 1.50 & 2.29 & 2.76 & 1.97 & \textbf{1.93} \\
				Accel $\downarrow$ & 1.13 & 1.45 & 1.22 & 1.49 & 1.31 & 1.70 & 1.28 & 1.61 & 1.31 & 1.88 & 1.24 & 1.14 & 1.73 & 1.93 & 1.47 & \textbf{1.47} \\
				BLC $\downarrow$ & 5.27 & 5.58 & 5.83 & 5.98 & 5.92 & 8.66 & 5.53 & 6.33 & 7.12 & 10.47 & 6.34 & 6.07 & 7.22 & 7.78 & 7.16 & \textbf{6.86} \\
				$\mathcal{D}(\mathcal{T})$ & 9.87 & 9.93 & 9.94 & 9.89 & 9.86 & 9.91 & 9.86 & 9.90 & 9.89 & 9.89 & 9.89 & 9.92 & 9.91 & 9.95 & 9.95 & \textbf{9.91} \\
				MAP $\uparrow$ & 71.76\% & 73.41\% & 73.30\% & 72.19\% & 70.14\% & 73.55\% & 68.85\% & 70.23\% & 72.20\% & 74.90\% & 70.59\% & 76.14\% & 74.31\% & 74.54\% & 80.42\% & \textbf{73.54\%} \\
				$\mathcal{H}$ $\downarrow$ & 0.87 & 0.78 & 0.88 & 0.93 & 0.95 & 0.78 & 0.94 & 0.81 & 0.92 & 0.98 & 0.96 & 0.78 & 0.84 & 0.67 & 0.73 & \textbf{0.86} \\
				\bottomrule
			\end{tabular}
		}
	\end{table*}


\end{document}